\documentclass{article}

\PassOptionsToPackage{numbers, compress}{natbib}

\usepackage[final]{neurips_2019}
 \usepackage{amsmath}
\usepackage{amsmath,amssymb,amsthm}
\usepackage{booktabs} 
\usepackage{amscd}
\usepackage{mathrsfs} 
\usepackage[shortlabels]{enumitem}
\usepackage{euscript}
\usepackage{graphicx}
\usepackage{subfig}
\usepackage{amscd,bm}
\RequirePackage[colorlinks,citecolor=blue,urlcolor=blue]{hyperref}
\usepackage[utf8]{inputenc}
\usepackage[english]{babel}
\usepackage{multicol}
\usepackage{multirow}
\usepackage{calrsfs}

\DeclareMathAlphabet{\pazocal}{OMS}{zplm}{m}{n}
\DeclareMathAlphabet\mathbfcal{OMS}{cmsy}{b}{n}

\usepackage{algorithm}
\usepackage{algorithmic}
\newtheorem{theorem}{Theorem}

\newtheorem{lemma}{Lemma}

\newtheorem{remark}{Remark}

\newtheorem{corollary}{Corollary}

\providecommand{\nor}[1]{\ensuremath{\left\lVert {#1} \right\rVert}}

\providecommand{\scalT}[2]{\ensuremath{\left\langle{#1},{#2}\right\rangle}}

\newcommand{\indep}{\raisebox{0.05em}{\rotatebox[origin=c]{90}{$\models$}}}

\def\bit{\begin{itemize}}
\def\eit{\end{itemize}}

\usepackage{mathtools}

\def\ben{\begin{enumerate}}
\def\een{\end{enumerate}}

\usepackage{listings}
\usepackage{color}

\definecolor{dkgreen}{rgb}{0,0.6,0}
\definecolor{gray}{rgb}{0.5,0.5,0.5}
\definecolor{mauve}{rgb}{0.58,0,0.82}

\definecolor{cb_orange}{rgb}{1.0,0.51,0.0}
\definecolor{cb_blue}{rgb}{0.22,0.49,0.72}
\definecolor{cb_green}{rgb}{0.3,0.67,0.29}
\definecolor{cb_red}{rgb}{0.89,0.1,0.11}
\definecolor{cb_purple}{rgb}{0.6, 0.31, 0.64}
\definecolor{cb_brown}{rgb}{0.6, 0.4, 0.2}
\definecolor{cb_crimson}{rgb}{0.86, 0.08, 0.24}

\usepackage{enumitem}

\usepackage[utf8]{inputenc} 
\usepackage[T1]{fontenc}    
\usepackage{hyperref}       
\usepackage{url}            
\usepackage{booktabs}       
\usepackage{amsfonts}       
\usepackage{nicefrac}       
\usepackage{microtype}      

\title{Sobolev Independence Criterion}

\author{%
  Youssef Mroueh, Tom Sercu, Mattia Rigotti, Inkit Padhi, Cicero Dos Santos \thanks{ Tom Sercu is now with Facebook AI Research, and Cicero Dos Santos with Amazon AWS AI. The work was done when they were at IBM Research.} \\
  IBM Research $\&$
  MIT-IBM Watson AI lab \\
 \texttt{mroueh,mrigotti@us.ibm.com,inkit.padhi@ibm.com} \\
}

\begin{document}

\maketitle

\begin{abstract}
We propose the Sobolev Independence Criterion (SIC), an interpretable dependency measure between a high dimensional random variable $X$ and a response variable $Y$.
SIC decomposes to the sum of feature importance scores and hence can be used for nonlinear feature selection.
SIC can be seen as a gradient regularized Integral Probability Metric (IPM) between the joint distribution of the two random variables and the product of their marginals.
We use sparsity inducing gradient penalties to promote input sparsity of the critic of the IPM.
In the kernel version we show that SIC can be cast as a convex optimization problem by introducing auxiliary variables that play an important role in feature selection as they are normalized feature importance scores.
We then present a neural version of SIC where the critic is parameterized as a homogeneous neural network, improving its representation power as well as its interpretability.
We conduct experiments validating SIC for feature selection in synthetic and real-world experiments.
We show that SIC enables reliable and interpretable discoveries, when used in conjunction with the holdout randomization test and knockoffs to control the False Discovery Rate. Code is available at \url{http://github.com/ibm/sic}.
\end{abstract}

\section{Introduction}
Feature Selection is an important problem in statistics and machine learning for interpretable predictive modeling and scientific discoveries.
Our goal in this paper is to design a dependency measure that is interpretable and can be reliably used to control the False Discovery Rate in feature selection.
\noindent
The mutual information between two random variables $X$ and $Y$ is the most commonly used dependency measure.
The mutual information $I(X;Y)$ is defined as the Kullback-Leibler divergence between the joint distribution $p_{xy}$ of $X,Y$ and the product of their marginals $p_xp_y$, $I(X;Y)=\text{KL}(p_{xy},p_{x}p_{y})$.
Mutual information is however challenging to estimate from samples, which motivated the introduction of dependency measures based on other $f$-divergences or Integral Probability Metrics \citep{Sriperumbudur2009OnIP} than the $\text{KL}$ divergence.
For instance, the Hilbert-Schmidt Independence Criterion (HSIC) \citep{HSIC} uses the Maximum Mean Discrepancy (MMD) \citep{MMD} to assess the dependency between two variables, i.e.\ $\text{HSIC}(X,Y)= \text{MMD}(p_{xy},p_{x}p_{y})$,
which can be easily estimated from samples via Kernel mean embeddings in a Reproducing Kernel Hilbert Space (RKHS) \citep{KernelMeanEmbedding}. 
In this paper we introduce the Sobolev Independence Criterion (SIC), a form of gradient regularized Integral Probability Metric (IPM)~\cite{SobolevGAN,SobolevDescent,Arbel:2018} between the joint distribution and the product of  marginals.
SIC relies on the statistics of the gradient of a witness function, or critic, for both
(1) defining the IPM constraint and 
(2) finding the features that discriminate between the joint and the marginals.
Intuitively, the magnitude of the average gradient with respect to a feature gives an importance score for each feature. Hence, promoting its sparsity is a natural constraint for feature selection.

\noindent
The paper is organized as follows: we show in Section \ref{sec:gradPenalty} how sparsity-inducing gradient penalties can be used to define an interpretable dependency measure that we name Sobolev Independence Criterion (SIC).
We devise an equivalent computational-friendly formulation of SIC in Section \ref{sec:etatrick}, that gives rise to additional auxiliary variables $\eta_j$. These naturally define normalized feature importance scores that can be used for feature selection.
In Section \ref{sec:ConvexSIC} we study the case where the SIC witness function $f$ is restricted to an RKHS and show that it leads to an optimization problem that is jointly convex in $f$ and the importance scores $\eta$.
We show that in this case SIC decomposes into the sum of feature scores, which is ideal for feature selection.
In Section \ref{sec:NeuralSIC} we introduce a Neural version of SIC, which we show preserves the advantages in terms of interpretability when the witness function is parameterized as a homogeneous neural network, and which we show can be optimized using stochastic Block Coordinate Descent.
In Section \ref{sec:HRT} we show how SIC and conditional Generative models can be used to control the False Discovery Rate using the recently introduced Holdout Randomization Test \citep{HRT} and Knockoffs \citep{Candes2018PanningFG}.
We validate SIC and its FDR control on synthetic and real datasets in Section \ref{sec:exp}.
    
    
    
    
 
\section{Sobolev Independence Criterion: Interpretable Dependency Measure}\label{sec:gradPenalty}

\noindent \textbf{Motivation: Feature Selection.}
We start by motivating gradient-sparsity regularization in SIC as a mean of selecting the features that maintain maximum dependency between two randoms variable $X$ (the input) and $Y$ (the response) defined on two spaces $\pazocal{X} \subset \mathbb{R}^{d_{x}}$ and $\pazocal{Y}\subset \mathbb{R}^{d_{y}}$ (in the simplest case $d_y=1$).  Let $p_{xy}$ be the joint distribution of $(X,Y)$ and $p_{x}$, $p_{y}$ be  the marginals of $X$ and $Y$ resp. 
Let $D$ be an Integral Probability Metric associated with a function space $\mathcal{F}$, i.e for two distributions $p,q$: 
$$D(p,q)=\sup_{f\in \mathcal{F}} \mathbb{E}_{x\sim p}f(x)-\mathbb{E}_{x\sim q}f(x).$$
With $p=p_{xy}$ and $q=p_x p_y$ this becomes a generalized definition of Mutual Information. Instead of the usual KL divergence, the metric $D$ with its witness function, or critic, $f(x,y)$ measures the distance between the joint $p_{xy}$ and the product of marginals $p_x p_y$.
With this generalized definition of mutual information,
the feature selection problem can be formalized as finding a sparse selector or gate $w \in \mathbb{R}^{d_x}$ such that $ D(p_{w\odot x,y},p_{w\odot x}p_y)$  is maximal \citep{mmdFeatureSelection, Feng2017SparseInputNN,penalizedNNDropOne,gating}
, i.e.\
$\sup_{w, \nor{w}_{\ell_0}\leq s } D(p_{w\odot x,y},p_{w\odot x}p_y),$
where $\odot$ is a pointwise multiplication and $\nor{w}_{\ell_0}=\#\{j|w_j \neq 0 \}$. This problem can be written in the following penalized form:
\vskip -0.2in
$$ \text{(P)}:~~ \sup_{w} \sup_{f \in \mathcal{F}} \mathbb{E}_{p_{xy}} f(w\odot x, y)-  \mathbb{E}_{p_{x}p_{y}} f(w\odot x, y)  - \lambda ||w||_{\ell_0} .$$
\vskip -0.12in

\noindent
We can relabel $\tilde{f}(x,y)= f(w\odot x, y)$ and write (P)  as: 
$  \sup_{\tilde{f} \in\tilde{ \mathcal{F}}} \mathbb{E}_{p_{xy}} \tilde{f} (x, y)-  \mathbb{E}_{p_{x}p_{y}} \tilde{f}( x, y)  $, 
where $\tilde{\mathcal{F}}=\{ \tilde{f}|\tilde{f}(x,y)= f(w\odot x,y) | f\in \mathcal{F}, \nor{w}_{\ell_0} \leq s\}$.
Observe that we  have:
$\frac{\partial \tilde{f} }{\partial x_j}= w_j \frac{\partial f(w\odot x,y)}{\partial x_j}.$
Since $w_j$ is sparse the gradient of $\tilde{f}$ is sparse on the support of $p_{xy}$ and $p_{x}p_{y}$. 
Hence, we can reformulate the problem (P) as follows: 
$$\text{(SIC):~~} \sup_{f \in \mathcal{F}} \mathbb{E}_{p_{xy}} f(x,y)- \mathbb{E}_{p_{x}p_{y}}f(x,y) - \lambda P_{S}(f),$$
where $P_{S}(f)$ is a penalty that controls the sparsity of the gradient of the witness function $f$ on the support of the measures. Controlling the nonlinear sparsity of the witness function in (SIC) via its gradients is more general and powerful than the linear sparsity control suggested in the initial form (P), since it takes into account the nonlinear interactions with other variables.  
In the following Section we formalize this intuition by theoretically examining sparsity-inducing gradient penalties \citep{sparsityK}.

%

\noindent \textbf{Sparsity Inducing Gradient Penalties.}
Gradient penalties have a long history in machine learning and signal processing. In image processing the total variation norm is used for instance as a regularizer to induce smoothness. Splines in Sobolev spaces \citep{wahba1975smoothing}, and manifold learning exploit gradient regularization to promote smoothness and regularity of the estimator.
In the context of neural networks, gradient penalties were made possible through double back-propagation introduced in \citep{drucker1992improving} and were shown to promote robustness and better generalization.  Such smoothness penalties became popular in deep learning partly following the introduction of WGAN-GP \citep{gulrajani2017improved},
and were used as regularizer for distance measures between distributions in connection to optimal transport theory \citep{SobolevGAN,SobolevDescent,Arbel:2018}.
Let $\mu$ be a dominant measure of $p_{xy}$ and $p_{x}p_{y}$  the most  commonly used gradient penalties is 
$$\Omega_{L^2}(f)= \mathbb{E}_{(x,y) \sim \mu} \nor{\nabla_x f(x,y)}^2.$$
While this penalty promotes smoothness, it does not control the desired sparsity as discussed in the previous section. We therefore elect to instead use the nonlinear sparsity penalty introduced in \citep{sparsityK} :
$\Omega_{\ell_0}(f)=\#\{j | \mathbb{E}_{(x,y)\sim \mu}\left|\frac{\partial f(x,y)}{\partial x_j}\right|^2 = 0 \}$, and its relaxation :
$$ \Omega_{S}(f)= \sum_{j=1}^{d_x} \sqrt{\mathbb{E}_{(x,y)\sim \mu}\left|\frac{\partial f(x,y)}{\partial x_j}\right|^2}.$$
\vskip -0.15in

\noindent
As discussed in \citep{sparsityK}, $\mathbb{E}_{(x,y)\sim \mu}\left|\frac{\partial f(x,y)}{\partial x_j}\right|^2=0$ implies that $f$ is constant with respect to variable $x_j$, if the function $f$ is continuously differentiable and the support of $\mu$ is connected.
These considerations motivate the following definition of the \emph{Sobolev Independence Criterion} (SIC): 
$$ \text{SIC}_{(L_1)^2}(p_{xy},p_{x}p_{y}) = \sup_{f\in \mathcal{F}} \mathbb{E}_{p_{xy}} {f} (x, y)-  \mathbb{E}_{p_{x}p_{y}} f( x, y) - \frac{\lambda}{2} \left(  \Omega_{S}(f) \right)^2- \frac{\rho}{2} \mathbb{E}_{\mu}f^2(x,y).$$
Note that we add a $\ell_1$-like penalty ($ \Omega_{S}(f)$ ) to ensure sparsity and an $\ell_2$-like penalty ($\mathbb{E}_{\mu}f^2(x,y)$) to ensure stability. This is similar to practices with linear models such as Elastic net.

\noindent
Here we will consider $\mu=p_{x}p_{y}$ (although we could also use $\mu= \frac{1}{2}(p_{xy}+p_{x}p_{y})$). Then, given samples $\{(x_i,y_i), i=1,\dots,N\}$ from the joint probability distribution $p_{xy}$ and \emph{iid} samples $\{(x_i,\tilde{y}_i), i=1,\dots,N\}$ from $p_{x}p_{y}$, SIC can be estimated as follows:
\vskip -0.2 in 
$$\widehat{\text{SIC}}_{(L_1)^2}(p_{xy},p_{x}p_{y}) = \sup_{f\in \mathcal{F}}  \frac{1}{N}\sum_{i=1}^N f (x_i, y_i)-  \frac{1}{N}\sum_{i=1}^Nf( x_i, \tilde{y}_i) - \frac{\lambda}{2} \left( \hat{ \Omega}_{S}(f) \right)^2- \frac{\rho}{2}\frac{1}{N}\sum_{i=1}^{N}f^2(x_i,\tilde{y}_i), $$
where
$\hat{ \Omega}_{S}(f) = \sum_{j=1}^{d_x} \sqrt{\frac{1}{N}\sum_{i=1}^{N} \left|\frac{  \partial f(x_i,\tilde{y}_i)}{\partial x_j} \right|^2 }. $

%
%
\begin{remark}
Throughout this paper we consider feature selection only on $x$ since $y$ is thought of as the response. Nevertheless, in many other problems one can perform feature selection on $x$ and $y$ jointly, which can be simply achieved by also controlling the sparsity of $\nabla_{y}f(x,y)$ in a similar way. 
\end{remark}

\section{Equivalent Forms of SIC with $\eta$-trick}\label{sec:etatrick}
As it was just presented, the SIC objective is a difficult function to optimize in practice. First of all, the expectation appears after the square root in the gradient penalties, resulting in a non-smooth term (since the derivative of square root is not continuous at 0). Moreover, the fact that the expectation is inside the nonlinearity introduces a gradient estimation bias when the optimization of the SIC objective is performed using stochastic gradient descent (i.e.\ using mini-batches).
We alleviate these problems (non-smoothness and biased expectation estimation) by making the expectation linear in the objective thanks to the introduction of auxiliary variables $\eta_j$ that will end up playing an important role in this work.
This is achieved thanks to a variational form of the square root that is derived from the following Lemma (which was used for a similar purpose as ours when alleviating the non-smoothness of mixed norms encountered in  multiple kernel learning and group sparsity norms): 

\begin{lemma} [\citep{Argyriou:2008},\citep{Bach:2011}]
Let $a_j,j=1\dots d$, $a_j> 0$ we have:
$\left(\sum_{j=1}^d \sqrt{a_j}\right)^2= \inf \{\sum_{j=1}^d \frac{a_j}{\eta_j}: \eta, \eta_j > 0 \sum_{j=1}^d \eta_j=1\},$
optimum achieved at $\eta_j = \sqrt{a_j}/\sum_j \sqrt{a_j}$.
\label{lem:etatrick}
\end{lemma}

\noindent  We alleviate first the issue of non smoothness of the square root  by adding an $\varepsilon \in (0,1)$, and we define:
$\Omega_{S,\varepsilon}=  \sum_{j=1}^{d_x} \sqrt{\mathbb{E}_{(x,y)\sim \mu}\left|\frac{\partial f(x,y)}{\partial x_j}\right|^2+\varepsilon}  $.
Using Lemma \ref{lem:etatrick}  the nonlinear sparsity inducing  gradient penalty can be written as :
\vskip -0.3in
\begin{eqnarray*}
(\Omega_{S,\varepsilon}(f))^2=\inf\{      \sum_{j=1}^{d_x} \frac{\mathbb{E}_{p_x p_y} \left|\frac{\partial f(x,y)}{\partial x_j }\right|^2  +\varepsilon}{\eta_j}: \eta,\eta_j>0,\sum_{j=1}^{d_{x}} \eta_j=1\},
\end{eqnarray*}
\vskip -0.1in
\noindent where the optimum is achieved for :
$\eta_{j,\varepsilon}^*
= \frac{\beta_j}{\sum_{k=1}^{d_x} \beta_{k}},$
where $\beta^2_j=\mathbb{E}_{p_{x}p_{y}}  \left|\frac{\partial f(x,y)}{\partial x_j }\right|^2+\varepsilon$.
We refer to $\eta_{j,\varepsilon}^*$ as the normalized importance score of feature $j$. Note that $\eta_j$ is a distribution over the  features and gives a natural ranking between the features.
Hence, substituting $\Omega(S)(f)$ with $\Omega_{S,\varepsilon}(f)$ in its equivalent form we obtain the $\varepsilon$ perturbed SIC:
\vskip -0.2in
$$\text{SIC}_{(L_1)^2,\varepsilon}(p_{xy},p_{x}p_{y})=-\inf\{ L_{\varepsilon}(f,\eta): f \in \mathcal{F},\eta_j, \eta_j >0 , \sum_{j=1}^{d_{x}} \eta_j=1  \} $$
\vskip -0.2in
\noindent where $ L_{\varepsilon}(f,\eta)= -\Delta(f,p_{xy},p_{x}p_{y})+\frac{ \lambda}{2} \sum_{j=1}^{d_x} \frac{\mathbb{E}_{p_x p_y} \left|\frac{\partial f(x,y)}{\partial x_j }\right|^2+\varepsilon }{\eta_j} + \frac{\rho}{2} \mathbb{E}_{p_{x}p_{y}}f^2(x,y),
$ and $\Delta(f,p_{xy},p_{x}p_{y})=\mathbb{E}_{p_{xy}} {f} (x, y)-  \mathbb{E}_{p_{x}p_{y}} f( x, y) $.
\noindent Finally, SIC can be empirically estimated as
$$\widehat{\text{SIC}}_{(L_1)^2,\varepsilon}(p_{xy},p_{x}p_{y}) = 
-\inf \{ \hat{L}_{\varepsilon}(f,\eta): f \in \mathcal{F},\eta_j, \eta_j >0 , \sum_{j=1}^{d_{x}} \eta_j=1 \}$$
\vskip -0.2in
\noindent where $\hat{L}_{\varepsilon}(f,\eta)=-\hat{\Delta}(f,p_{xy},p_xp_y)   + \frac{\lambda}{2} \sum_{j=1}^{d_x} \frac{  \frac{1}{N} \sum_{i=1}^{N} \left|\frac{\partial f(x_i,\tilde{y}_i)}{\partial x_j }\right|^2 +\varepsilon }{\eta_j} + \frac{\rho}{2}\frac{1}{N}\sum_{i=1}^{N}f^2(x_i,\tilde{y}_i), $
and main the objective $\hat{\Delta}(f,p_{xy},p_{x}p_{y})=\frac{1}{N}\sum_{i=1}^N f (x_i, y_i)-  \frac{1}{N}\sum_{i=1}^Nf( x_i, \tilde{y}_i) $.


\begin{remark}[Group Sparsity] We can define similarly nonlinear group sparsity, if we would like our critic to depends on subsets of coordinates. 
Let $G_k, k=1,\dots,K $ be an  overlapping or non overlapping group : 
$\Omega_{gS}(f)= \sum_{k=1}^{K}\sqrt{ \sum_{j \in G_k}\mathbb{E}_{p_{x}p_{y}}  \left|\frac{\partial f(x,y)}{\partial x_j }\right|^2}$. The $\eta$-trick applies naturally. 
\end{remark}

\section{Convex Sobolev  Independence Criterion in Fixed Feature Spaces }\label{sec:ConvexSIC}
We will now specify the function space $\mathcal{F}$ in SIC and  consider in this Section critics of the form:
$$\mathcal{F}=\{ f|  f(x,y)= \scalT{u}{\Phi_{\omega}(x,y)}, \nor{u}_2\leq \gamma \},$$
where $\Phi_{\omega}: \pazocal{X}\times \pazocal{Y}\to \mathbb{R}^m$ is a fixed finite dimensional feature map.
We define the mean embeddings of the joint distribution $p_{xy}$ and product of marginals $p_{x}p_y$ as follow:
$\mu(p_{xy})= \mathbb{E}_{p_{xy}} [\Phi_{\omega}(x,y)], ~~\mu(p_{x}p_{y})= \mathbb{E}_{p_{x}p_{y}}[\Phi_{\omega}(x,y)] \in \mathbb{R}^m .$
Define the covariance  embedding of $p_{x}p_{y}$ as
$C(p_{x}p_{y})= \mathbb{E}_{p_{x}p_{y}}[\Phi_{\omega}(x,y)\otimes \Phi_{\omega}(x,y)] \in \mathbb{R}^{m\times m}$
and finally define the Gramian of derivatives embedding for coordinate $j$ as
$D_j(p_{x}p_{y})= \mathbb{E}_{p_{x}p_{y}} [\frac{ \partial \Phi_{\omega}(x,y)}{\partial x_j } \otimes\frac{ \partial \Phi_{\omega}(x,y)}{\partial x_j }] \in  \mathbb{R}^{m\times m}. $
\noindent We can write the constraint $\nor{u}_2\leq \gamma$ as the penalty term $-\tau \nor{u}^2$. 
Define  
$L_{\varepsilon}(u,\eta)= \scalT{u}{\mu(p_xp_y)- \mu(p_{xy})} + \frac{1}{2}\scalT{u}{\left( \lambda \sum_{j=1}^{d_x} \frac{D_{j}(p_{x}p_{y})+\varepsilon }{\eta_j} + \rho C(p_xp_y)+ \tau I_m \right) u}$.
Observe that : 
\vskip -0.2in
$$\text{SIC}_{(L^1)^2,\varepsilon}(p_{xy},p_{x}p_{y}) = - \inf \{ L_{\varepsilon}(u,\eta): u \in \mathbb{R}^m, \eta_j, \eta_j >0 , \sum_{j=1}^{d_{x}} \eta_j=1  \}.$$
\vskip -0.1in
\vskip -0.05in
\noindent We start by remarking that SIC is a form of gradient regularized maximum mean discrepancy \citep{MMD}.
Previous MMD work comparing joint and product of marginals did not use the concept of nonlinear sparsity. For example the Hilbert-Schmidt Independence Criterion (HSIC) \cite{HSIC} uses $\Phi_{\omega}(x,y)=\phi(x)\otimes \psi(y)$ with a constraint $||u||_{2}\leq 1$. CCA and related kernel measures of dependence \cite{VINOD1976147,NIPS2007_3340} use $L^2_2$ constraints $L^2_2(p_x)$ and $L^2_2(p_y)$ on each function space separately.

\noindent \textbf{Optimization Properties of Convex  SIC}
We analyze in this Section the Optimization properties of SIC. Theorem \ref{theo:ConvexityAndSmoothing} shows that the $\text{SIC}_{(L^1)^2,\varepsilon}$ loss function is jointly strictly convex in $(u,\eta)$ and hence admits a unique solution that solves a fixed point problem. 
\begin{theorem} [Existence of a solution, Uniqueness, Convexity and Continuity] Note that $L(u,\eta)=L_{\varepsilon=0}(u,\eta)$.
The following properties hold for the SIC loss:

\noindent 1) $L(u,\eta)$ is differentiable and  jointly convex in $(u,\eta)$. $L(u,\eta)$ is not continuous  for  $\eta$, such that $\eta_j=0$ for some $j$.

\noindent 2) Smoothing, Perturbed SIC:  For $\varepsilon \in (0,1)$, $L_{\varepsilon}(u,\eta)= L(u,\eta) + \frac{\lambda}{2} \sum_{j=1}^{d_x}\frac{\varepsilon}{\eta_j}$ is   jointly strictly convex and has compact level sets  on the probability simplex, and admits a unique minimizer $(u^*_{\varepsilon},\eta^*_{\varepsilon})$.

\noindent 3) The unique minimizer of $L_{\varepsilon}(u,\eta)$  is a solution of the following fixed point problem:
$u^*_{\varepsilon}= \left( \lambda \sum_{j=1}^{d_x} \frac{D_{j}(p_{x}p_{y}) }{\eta^*_j} + \rho C(p_xp_y)+ \tau I_m \right)^{-1}(\mu(p_{xy}) - \mu(p_xp_y) )$,
and 
$\eta_{j,\varepsilon}^* = \frac{\sqrt{\scalT{u^*_{\varepsilon}}{D_j(p_x p_y) u^*_{\varepsilon}}+\varepsilon}}{ \sum_{k=1}^{d_x}\sqrt{\scalT{u^*_{\varepsilon}}{D_k(p_{x}p_{y}) u^*_{\varepsilon}}+\varepsilon}}.$

\label{theo:ConvexityAndSmoothing}
\end{theorem}

\noindent The following Theorem shows that a solution of the unperturbed SIC problem can be obtained from the smoothed $\text{SIC}_{(L_1)^2,\varepsilon}$ in the limit $\varepsilon \to 0$:

\begin{theorem}[From Perturbed SIC to SIC]
Consider a sequence $\varepsilon_{\ell}$, $\varepsilon_{\ell}\to 0$ as $\ell \to \infty$ , and consider a sequence of minimizers $(u^*_{\varepsilon_{\ell}},\eta^*_{\ell})$  of $L_{\varepsilon_{\ell}}(u,\eta)$, and let $(u^*,\eta^*)$ be the limit of this sequence, then $(u^*,\eta^*)$ is a minimizer of $L(u,\eta)$.
\label{theo:SICoptim}
\end{theorem}

\noindent \textbf{Interpretability of SIC.}
The following corollary shows that SIC can be written in terms of the importance scores of the features, since at optimum the main objective is proportional to  the constraint term.
It is  to the best of our knowledge the first dependency criterion that decomposes in the sum of contributions of each coordinate, and hence it is an interpretable dependency measure.
Moreover, $\eta^*_j$ are normalized importance scores of  each feature $j$, and their ranking can be used to assess feature importance.

\begin{corollary}[Interpretability of Convex SIC ]  Let $(u^*,\eta^*)$ be the limit defined in Theorem  \ref{theo:SICoptim}. Define $f^*(x,y)=\scalT{u^*}{\Phi_{\omega}(x,y)}$, and $\nor{f^*}_{\mathcal{F}}=\nor{u^*}$. We have that
\begin{eqnarray*}
 \text{SIC}_{(L^1)^2} (p_{xy},p_xp_y) &=& \frac{1}{2}\left(\mathbb{E}_{p_{xy}} f^*(x,y) -  \mathbb{E}_{p_{x}p_{y}} f^*(x,y)\right)\\
 & =& \frac{\lambda}{2} \left(\sum_{j=1}^{d_x} \sqrt{\mathbb{E}_{p_xp_y} |\frac{\partial f^*(x,y)}{\partial x_j}|^2 }\right)^2 +\frac{\rho}{2} \mathbb{E}_{p_xp_y} f^{*,2}(x,y) + \frac{\tau}{2} ||f^*||^2_{\mathcal{F}}.
 \end{eqnarray*}
Moreover, 
$\sqrt{  \mathbb{E}_{p_xp_y}  |\frac{\partial f^*(x,y)}{\partial x_j}|^2}  = \eta^*_{j} \Omega_{S,L_1}(f^*) \text{ and } \sum_{j=1}^{d_x} \eta_j=1 .$
The terms $\eta_j^*$  can be seen as quantifying how much dependency as measured by SIC can be explained by a coordinate $j$. Ranking of $\eta^*_j$ can be used to rank influence of coordinates. 
 \label{cor:InterpretableSIC} 
\end{corollary}

\noindent Thanks to the joint convexity and  the smoothness of the perturbed SIC, we can solve convex empirical SIC using alternating minimization on $u$ and $\eta$ or block coordinate descent using first order methods such as gradient descent on $u$ and mirror descent \citep{Beck:2003} on $\eta$ that are known to be globally convergent in this case (see Appendix \ref{alg:ConvexSIC} for more details).  



\section{Non Convex Neural SIC with Deep ReLU Networks}
\label{sec:NeuralSIC}
While Convex SIC enjoys a lot of theoretical properties, a crucial short-coming  is the need to choose a feature map $\Phi_{\omega}$ that essentially goes back to the choice of a kernel in classical kernel methods.
As an alternative, we propose to learn the feature map as a deep neural network. The architecture of the network can be problem dependent, but we focus here on a particular architecture: Deep ReLU Networks with biases removed. As we show below, using our sparsity inducing gradient penalties with such networks, results in input sparsity at the level of the witness function $f$ of SIC. This is desirable since it allows for an interpretable model, similar to the effect of  Lasso with Linear models, our sparsity inducing gradient penalties result in a nonlinear self-explainable witness function $f$ \citep{selfexplainable}, with explicit sparse dependency on the inputs.
 
\noindent \textbf{Deep ReLU Networks with no biases, homogeneity and Input Sparsity via Gradient Penalties.}
We start by invoking the Euler Theorem for homogeneous functions:
\vskip -0.4in
\begin{theorem}[Euler Theorem for Homogeneous Functions] A continuously differentiable function $f$ is defined as homogeneous of degree $k$ if $f(\lambda x)=\lambda^k f(x), \forall \lambda \in \mathbb{R}$. The Theorem states that $f$ is homogeneous of degree $k$ if and only if
$k f(x)= \scalT{\nabla_xf(x)}{x}=\sum_{j=1}^{d_{x}}\frac{\partial f(x)}{\partial x_j}x_j.$
\end{theorem}
\vskip -0.1 in 

\noindent Now consider \emph{deep ReLU networks with biases removed} for any number of layers $L$:
 $\mathcal{F}_{ReLu}=\{ f| f(x,y)=\scalT{u}{\Phi_{\omega}(x)}, \text{ where } \Phi_{\omega}(x,y)= \sigma (W_L \dots \sigma(W_2\sigma(W_1[x,y]))), u \in \mathbb{R}^m, \Phi_{\omega}: \mathbb{R}^{d_x+d_{y}}\to \mathbb{R}^m \},$ 
 where $\sigma(t)=\max(t,0), W_j$ are linear weights.
 Any  $f \in \mathcal{F}_{ReLU}$ is clearly homogeneous of degree $1$. As an immediate consequence of Euler Theorem we then have: $f(x,y)=\scalT{\nabla_x f(x,y)}{x} + \scalT{\nabla_y f(x,y)}{y}$. The first term is similar to a linear term in a linear model, the second term can be seen as a bias. Using our sparsity-inducing gradient penalties with such networks guarantees that on average on the support of a dominant measure the gradients with respect to $x$ are sparse. Intuitively, the gradients wrt $x$ act like the weight in linear models, and our sparsity inducing gradient penalty act like the $\ell_1$ regularization of Lasso. The main advantage compared to Lasso is that we have a highly nonlinear decision function, that has better capacity of capturing dependencies between $X$ and $Y$.



\noindent \textbf{Non-convex SIC with Stochastic Block Coordinate Descent (BCD).}
We define the empirical non convex $SIC_{(L^1)^2}$ using this function space $\mathcal{F}_{\text{ReLu}}$ as follows:
\vskip -0.31in 
 \begin{align*}
&\widehat{\text{SIC}}_{(L^1)^2}(p_{xy},p_{x}p_{y})= - \inf\{ \hat{L}(f_{\theta},\eta): f_{\theta} \in \bm{\mathcal{F}_{ReLU}}, \eta_j, \eta_j >0 , \sum_{j=1}^{d_{x}} \eta_j=1\},
\end{align*}
\vskip -0.20in 
\noindent where $\theta=(vec(W_1)\dots vec(W_L),u )$  are the network parameters.
Algorithm \ref{alg:sBCD} in Appendix \ref{app:SICNeural} summarizes our stochastic BCD algorithm for training the Neural SIC. 
The algorithm consists of SGD updates to $\theta$ and mirror descent updates to $\eta$.

\noindent \textbf{Boosted SIC.} When training Neural SIC, we can obtain different critics $f_{\ell}$ and importance scores $\eta_{\ell}$, by varying random seeds or hyper-parameters (architecture, batch size etc). Inspired by importance scores in random forest, we define \textbf{Boosted SIC} as the arithmetic mean or the geometric mean of $\eta_{\ell}$.


\section{FDR Control and the Holdout Randomization Test/ Knockoffs.}\label{sec:HRT}
 Controlling the False Discovery Rate (FDR) in Feature Selection is an important problem for reproducible discoveries. 
 In a nutshell, for a feature selection problem given the ground-truth set of features $\pazocal{S}$, and a feature selection method such as SIC that gives a candidate set $\hat{S}$, our goal is to maximize the TPR (True Positive Rate) or the power, and to keep the False Discovery Rate (FDR) under Control. TPR and FDR are defined as follows:  
\begin{equation}
\label{eqn:power}
\text{TPR} \coloneqq \mathbb{E}\left[\frac{\#\{ i : i \in \hat{S} \cap \pazocal{S} \}}{\#\{ i : i \in \pazocal{S}\}}\right] ~~ \text{FDR} \coloneqq \mathbb{E}\left[\frac{\#\{ i : i \in \hat{{S}} \backslash \pazocal{S} \}}{\#\{ i : i \in \hat{{S}} \}}\right].
\end{equation}

\noindent We explore in this paper two methods that provably control the FDR: 1) The Holdout Randomization Test (HRT) introduced in \citep{HRT}, that we specialize for SIC in Algorithm \ref{alg:hrt_sic};
 2) Knockoffs introduced in \citep{Candes2018PanningFG} that can be used with any basic feature selection method such as Neural SIC, and guarantees provable FDR control.
 
\noindent \textbf{\textit{HRT-SIC.}} We are  interested in measuring the conditional dependency between a feature  $x_j$  and the response variable  $y$ conditionally on the other features noted $x_{-j}$. Hence we have the following null hypothesis: 
$H_0: x_j  \indep y ~| x_{-j} \iff p_{xy}=p_{x_j|x_{-j}}p_{y|x_{-j}}p_{x_{-j}}$.
In order to simulate the null hypothesis, we propose to use generative models for sampling from $x_j|x_{-j}$ (See Appendix \ref{sec:Generators}).    
The principle in HRT \citep{HRT} that we specify here for SIC  in Algorithm \ref{alg:hrt_sic} (given in Appendix \ref{app:SICNeural}) is the following: instead of refitting SIC under $H_0$, we evaluate the mean of the witness function of SIC on a holdout set sampled under $H_0$ (using conditional generators for $R$ rounds). The deviation of the mean of the witness function under $H_0$ from its mean on a holdout from the real distribution gives us $p$-values. We use the Benjamini-Hochberg  \citep{Benjamini:1995ram} procedure on those $p$-values to achieve a target FDR. We apply HRT-SIC on a shortlist of pre-selected features per their ranking of $\eta_j$.

\noindent \textbf{\textit{Knockoffs-SIC.}} Knockoffs \cite{barber2015controlling} work by finding control variables called knockoffs $\tilde{x}$ that mimic the behavior of the real features $x$ and provably control the FDR \citep{Candes2018PanningFG}. We use here Gaussian knockoffs \citep{Candes2018PanningFG} and train SIC on the concatenation of $[x,\tilde{x}]$, i.e we train $SIC([X;\tilde{X}],Y)$ and obtain $\eta$ that has now twice the dimension $d_x$, i.e for each real feature $j$, there is the real importance score $\eta_j$ and the knockoff importance score $\eta_{j+d_{x}}$. knockoffs-SIC consists in using the statistics $W_j=\eta_j-\eta_{j+d_x}$ and the knockoff filter \citep{Candes2018PanningFG} to select features based on the sign of $W_j$ (See Alg.\ \ref{alg:modelx_sic} in Appendix). 

\section{Relation to Previous Work}

\noindent \textbf{Kernel/Neural Measure of Dependencies.} As discussed earlier SIC can be seen as a  \emph{sparse} gradient regularized MMD \cite{MMD,Arbel:2018} and relates to the Sobolev Discrepancy of \citep{SobolevGAN,SobolevDescent}. Feature selection with MMD  was introduced in  \citep{mmdFeatureSelection} and is based on backward elimination of features by recomputing MMD on the ablated vectors. SIC has the advantage of fitting one critic that has interpretable feature scores. Related to the MMD is the Hilbert Schmidt Independence Criterion (HSIC) and other variants of kernel dependency measures introduced in \citep{HSIC,NIPS2007_3340}.
None of those criteria has a nonparametric sparsity constraint on its witness function that allows for explainability and feature selection. Other Neural measures of dependencies such as MINE \citep{mine} estimate the KL divergence using neural networks, or that of \citep{ozair2019wasserstein} that estimates a proxy to the Wasserstein distance using Neural Networks.

\noindent \textbf{Interpretability, Sparsity, Saliency and Sensitivity Analysis.}
Lasso and elastic net  \citep{hastie01statisticallearning} are interpretable linear models  that exploit sparsity, but are limited to linear relationships. Random forests  \citep{Breiman:2001} have a heuristic for determining  feature  importance and are successful in practice as they can capture nonlinear relationships similar to SIC. We believe SIC can potentially leverage the deep learning toolkit for going beyond tabular data where random forests excel, to more structured data such as time series or graph data. Finally, SIC relates to saliency based post-hoc interpretation of deep models such as \citep{GradInput,saliency,RelvPropagation}. While those method use the gradient information for a post-hoc analysis, SIC incorporates this information to guide the learning towards the important features. As discussed in Section 2.1 many recent works introduce deep networks with input sparsity control through a learned gate or a penalty on the weights of the network \citep{Feng2017SparseInputNN,penalizedNNDropOne,gating}. SIC exploits a stronger notion of sparsity that leverages the relationship between the different covariates.

%
%
%
%
%

\section{Experiments}\label{sec:exp}
\noindent \textbf{Synthetic Data Validation.} We first validate our methods and compare them to baseline models in simulation studies on synthetic datasets where the ground truth is available by construction.
For this we generate the data according to a model $y = f(x) + \epsilon$  where the model $f(\cdot)$ and the noise $\epsilon$ define the specific synthetic dataset (see Appendix \ref{subapp:datadetails}).
In particular, the value of $y$ only depends on a subset of features $x_i$, $i=1,\ldots,p$ through $f(\cdot)$, and performance is quantified in terms of TPR and FDR in discovering them among the irrelevant features.
We experiment with two datasets: 
\noindent \textbf{A) Complex multivariate synthetic data (SinExp)}, which is generated from a complex multivariate model proposed in \cite{feng2017sparse} Sec 5.3, where 6 \emph{ground truth} features $x_i$ out of 50 generate the output $y$ through a non-linearity involving the product and composition of the $\cos$, $\sin$ and $\exp$ functions (see Appendix \ref{subapp:datadetails}).
We therefore dub this dataset SinExp.
To increase the difficulty even further, we introduce a pairwise correlation between all features of 0.5.
In Fig.\ \ref{fig:sinexp} we show results for datasets of 125 and 500 samples repeated 100 times comparing performance of our models with the one of two baselines: Elastic Net (EN) and Random Forest (RF).
\noindent \textbf{B) Liang Dataset.}
We show results on the benchmark dataset proposed by \cite{liang2018bayesian}, specifically the \emph{generalized} Liang dataset matching most of the setup from \cite{HRT} Sec 5.1.
We provide dataset details and results in Appendix \ref{subapp:datadetails} (Results in Figure \ref{fig:liang}).

\begin{figure}[ht!]
  \centering
  \subfloat{\includegraphics[scale=0.23]{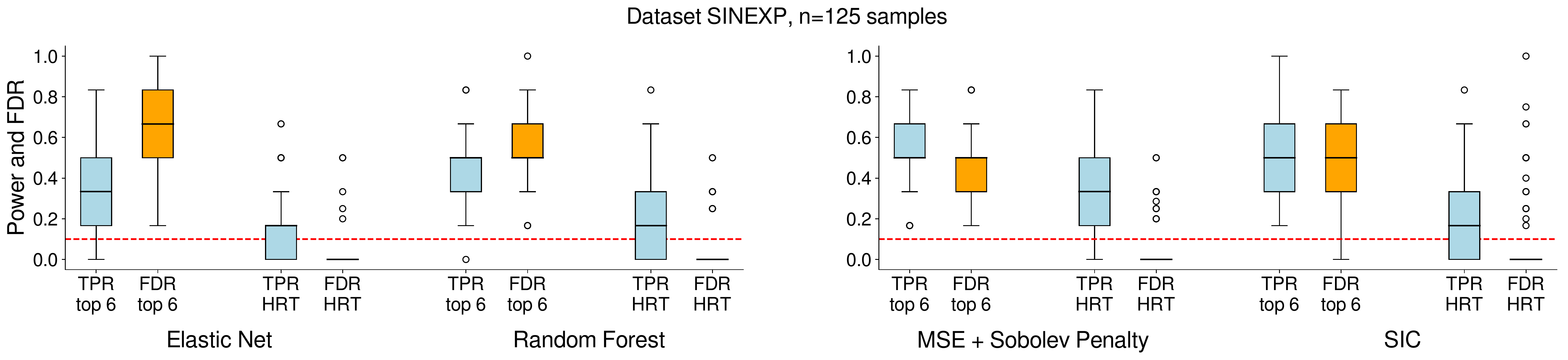}}\\
  \subfloat{\includegraphics[scale=0.23]{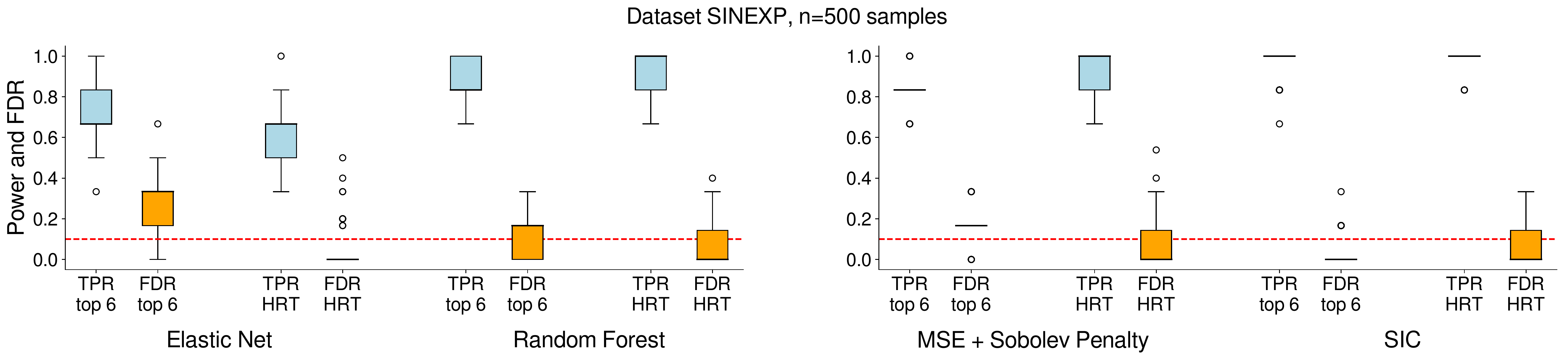}}\\
  \caption{SinExp synthetic dataset.
  TPR and FDR of Elastic Net (EN) and Random Forest (RF) baseline models (left panels) are compared to our methods: a 2-hidden layer neural network with no biases trained to minimize an objective comprising an MSE cost and a Sobolev Penalty term (MSE + Sobolev Penalty), and the same network trained to optimize SIC criterion (right panels), for datasets of 125 samples (top panels) and 500 samples (bottom panels).
  For all models TPR and FDR are computed by selecting the top 6 features in order of feature importance (which for EN is defined as the absolute value of the weight of a feature, for RF is the out-of-bag error associated to it (see \cite{breiman2001random}), and for our method is the value of its $\eta$).
  Selecting the first 6 features is useful to compare models, but assumes \emph{oracle knowledge} of the fact that there are 6 ground truth features.
  We therefore also compute FDR and TPR after selecting features using the HRT method of \cite{HRT} among the top 20 features.
  HRT estimates the importance of a feature quantifying its effect on the distribution of $y$ on a holdout set by replacing its values with samples from a conditional distribution (see Section \ref{sec:HRT}).
  We use HRT to control FDR rate at 10\% (red horizontal dotted line).
  Standard box plots are generated over 100 repetitions of each simulation.}
  \label{fig:sinexp}
\end{figure}


\noindent \textbf{Feature Selection on Drug Response dataset.}
We consider as a real-world application the Cancer Cell Line Encyclopedia (CCLE) dataset \cite{barretina2012cancer},
described in Appendix \ref{subapp:ccle}.
We study the result of using the normalized importance scores $\eta_j$ from SIC for (heuristic) feature selection, against features selected by Elastic Net.
Table \ref{tab:ccle} shows the heldout MSE of a predictor trained on selected features, averaged over 100 runs (each run: new randomized 90\%/10\% data split, NN initialization).
The goal here is to quantify the predictiveness of features selected by SIC on its own, without the full randomized testing machinery.
The SIC critic and regressor NN were respectively the $big\_critic$ and $regressor\_NN$ described with training details in Appendix \ref{app:nndetails}, while the random forest is trained with default hyper parameters from scikit-learn \citep{pedregosa2011scikit}. We can see that, with just $\eta_j$, informative features are selected for the downstream regression task, with performance comparable to those selected by ElasticNet, which was trained explicitly for this task. The features selected with high $\eta_j$ values and their overlap with the features selected by ElasticNet are listed in Appendix \ref{subapp:ccle} Table \ref{tab:ccle_ftrs}.
\label{sec:ccle}
\vskip -0.1in

\begin{table}[ht!]
    \centering
\begin{tabular}{lcc}
\toprule
{} &                 NN &                 RF \\
\midrule
All 7251 features         &  1.160 $\pm$ 3.990 &  0.783 $\pm$ 0.167 \\
Elastic-Net1 \cite{barretina2012cancer}  top-7 &  0.864 $\pm$ 0.432 &  0.931 $\pm$ 0.215 \\
Elastic-Net2 \cite{HRT} top-10   &  \bf{0.663 $\pm$ 0.161} &  0.830 $\pm$ 0.190 \\
SIC top-7              &  0.728 $\pm$ 0.166 &  0.856 $\pm$ 0.189 \\
SIC top-10             &  0.706 $\pm$ 0.158 &  \bf{0.817 $\pm$ 0.173} \\
SIC top-15             &  0.734 $\pm$ 0.168 &  0.859 $\pm$ 0.202 \\
\bottomrule
\end{tabular}
\caption{CCLE results on downstream regression task. Heldout MSE for drug PLX4720 prediction based on selected features. Columns: neural network (NN) and random forest (RF) regressors.}
    \label{tab:ccle}
\vskip -0.2in
\end{table}
\noindent \textbf{HIV-1 Drug Resistance with Knockoffs-SIC.}
The second real-world dataset that we analyze is the HIV-1 Drug Resistance\cite{rhee2006genotypic}, which consists in detecting mutations associated with resistance to a drug type. For our experiments we use all the three classes of drugs: Protease Inhibitors (PIs), Nucleoside Reverse Transcriptase Inhibitors (NRTIs), and Non-nucleoside Reverse Transcriptase Inhibitors (NNRTIs). We use the pre-processing of each dataset (<drug-class, drug-type>) of the knockoff tutorial \cite{hivStanfordTutorial} made available by the authors. Concretely, we construct a dataset $(X,\tilde{X})$ of the concatenation of the real data and Gaussian knockoffs~\citep{Candes2018PanningFG}, and fit $SIC([X,\tilde{X}],Y)$. As explained in Section \ref{sec:HRT}, we use in the knockoff filter the statistics $W_j=\eta_j-\eta_{j+d_x}$, i.e.\ the difference of SIC importance scores between each feature and its corresponding knockoff. For SIC experiments, we use $small\_critic$ architecture (See Appendix \ref{app:nndetails} for training details).
We use Boosted SIC, by varying the batch sizes in $N \in \{10, 30, 50\}$, and computing the geometric mean of $\eta$ produced by those three setups as the feature importance needed for Knockoffs. Results are summarized in Table \ref{tab:SICknockoff}. 
\vskip -0.05in
\begin{table}[ht!]
\centering
\begin{tabular}{l|c|rrr|rrr}
\toprule
Drug Class  & Drug Type & \multicolumn{3}{c}{Knockoff with GLM} & \multicolumn{3}{c}{Boosted SIC Knockoff} \\
\midrule
{} & {} & TD & FD & FDP & TD & FD & FDP \\
\midrule
\multirow{7}{*}{PIs} & APV   &  19 & 3 &  \textbf{0.13} &  17 & 5 & 0.22        \\  
                  & ATV      &    22 & 8 & 0.26      &    19 & 1 & \textbf{0.05}     \\  
                  &  IDV       &    19 & 12 & 0.38      &   15 & 3 & \textbf{0.16}      \\
                  &  LPV       &    16 & 1 & \textbf{0.05}      &    14 & 2 & 0.12     \\
                  &  NFV       &    24 & 7 & 0.22      &    19 & 5 & \textbf{0.21}     \\
                  &  RTV       &    19 & 8 & 0.29      &    12 & 2 & \textbf{0.20}     \\
                  &  SQV       &    17 & 4 & \textbf{0.19}      &    14 & 8 & 0.36     \\
\midrule
\multirow{5}{*}{NRTIs} &  X3TC &  0 & 0 & 0     & 7 & 0 & \textbf{0}        \\  
                  & ABC        & 10 & 1 & 0.09  & 11 & 1 & \textbf{0.08}     \\  
                  &   AZT      & 16 & 4 & \textbf{0.2}   & 12 & 5 & 0.29        \\ 
                  & D4T        & 6 & 1 & 0.14   &  8 & 0 & \textbf{0}        \\ 
                  &  DDI       & 0 & 0 & 0      &    8 & 0 & \textbf{0}     \\ 
\midrule
\multirow{3}{*}{NNRTIs} & DLV  &    10 & 13 & 0.56      &    8 & 10 & \textbf{0.55}     \\  
                  & EFV        &    11 & 11 & 0.5      &    11 & 10 & \textbf{0.47}     \\  
                  &  NVP       &    7 & 10 & \textbf{0.58}      &  7 & 11 & 0.611       \\
                   
\bottomrule
\end{tabular}
\caption{Comparison of applying (knockoff filter + GLM) and (Knockoff filter+Boosted SIC). For each <drug-class, drug-type> we compared the True Discoveries (TD), False Discoveries(FD) and False Discovery Proportion (FDP). Knockoff with Boosted SIC keeps FDP under control without compromising power, and succeeds in making true discoveries that GLM with knockoffs doesn't find.}
\label{tab:SICknockoff}
\vskip -0.2in
\end{table}

\section{Conclusion}
We introduced in this paper the Sobolev Independence Criterion (SIC), a dependency measure that gives rise to feature importance which can be used for feature selection and interpretable decision making.
We laid down the theoretical foundations of SIC and showed how it can be used in conjunction with the Holdout Randomization Test and Knockoffs to control the FDR, enabling reliable discoveries. We demonstrated the merits of SIC for feature selection in extensive synthetic and real-world experiments with controlled FDR.

\bibliographystyle{unsrt}
\bibliography{simplex,refs}

\begin{thebibliography}{10}

\bibitem{Sriperumbudur2009OnIP}
Bharath~K. Sriperumbudur, Kenji Fukumizu, Arthur Gretton, Bernhard Scholkopf,
  and Gert R.~G. Lanckriet.
\newblock On integral probability metrics, $\phi$-divergences and binary
  classification.
\newblock 2009.

\bibitem{HSIC}
A.~Gretton, K.~Fukumizu, CH. Teo, L.~Song, B.~Sch{\"o}lkopf, and AJ. Smola.
\newblock A kernel statistical test of independence.
\newblock In {\em Advances in neural information processing systems 20}, 2008.

\bibitem{MMD}
Arthur Gretton, Karsten~M. Borgwardt, Malte~J. Rasch, Bernhard Sch\"{o}lkopf,
  and Alexander Smola.
\newblock A kernel two-sample test.
\newblock {\em JMLR}, 2012.

\bibitem{KernelMeanEmbedding}
Krikamol Muandet, Kenji Fukumizu, Bharath Sriperumbudur, and Bernhard
  Sch\"{o}lkopf.
\newblock Kernel mean embedding of distributions: A review and beyond.
\newblock {\em Arxiv}, 2017.

\bibitem{SobolevGAN}
Youssef Mroueh, Chun-Liang Li, Tom Sercu, Anant Raj, and Yu~Cheng.
\newblock Sobolev gan.
\newblock {\em ICLR}, 2018.

\bibitem{SobolevDescent}
Youssef Mroueh, Tom Sercu, and Anant Raj.
\newblock Sobolev descent.
\newblock In {\em AISTATS}, 2019.

\bibitem{Arbel:2018}
Michael Arbel, Dougal~J. Sutherland, Mikolaj Binkowski, and Arthur Gretton.
\newblock On gradient regularizers for mmd gans.
\newblock {\em NeurIPS}, 2018.

\bibitem{HRT}
W.~Tansey, V.~Veitch, H.~Zhang, R.~Rabadan, and D.~M. Blei.
\newblock The holdout randomization test: Principled and easy black box feature
  selection.
\newblock {\em arXiv preprint arXiv:1811.00645}, 2018.

\bibitem{Candes2018PanningFG}
Emmanuel Candes, Yingying Fan, Lucas Janson, and Jinchi Lv.
\newblock Panning for gold: model-x knockoffs for high dimensional controlled
  variable selection.
\newblock 2018.

\bibitem{mmdFeatureSelection}
Le~Song, Alex Smola, Arthur Gretton, Justin Bedo, and Karsten Borgwardt.
\newblock Feature selection via dependence maximization.
\newblock {\em J. Mach. Learn. Res.}, 2012.

\bibitem{Feng2017SparseInputNN}
Jean Feng and Noah Simon.
\newblock Sparse-input neural networks for high-dimensional nonparametric
  regression and classification.
\newblock 2017.

\bibitem{penalizedNNDropOne}
Mao Ye and Yan Sun.
\newblock Variable selection via penalized neural network: a drop-out-one loss
  approach.
\newblock In {\em Proceedings of the 35th International Conference on Machine
  Learning}, 2018.

\bibitem{gating}
Yutaro Yamada, Ofir Lindenbaum, Sahand Negahban, and Yuval Kluger.
\newblock Deep supervised feature selection using stochastic gates.
\newblock {\em Arxiv}, 2018.

\bibitem{sparsityK}
Lorenzo Rosasco, Silvia Villa, Sofia Mosci, Matteo Santoro, and Alessandro
  Verri.
\newblock Nonparametric sparsity and regularization.
\newblock {\em J. Mach. Learn. Res.}, 2013.

\bibitem{wahba1975smoothing}
Grace Wahba.
\newblock Smoothing noisy data with spline functions.
\newblock {\em Numerische mathematik}, 24(4), 1975.

\bibitem{drucker1992improving}
Harris Drucker and Yann LeCun.
\newblock Improving generalization performance using double backpropagation.
\newblock {\em IEEE Transactions on Neural Networks}, 1992.

\bibitem{gulrajani2017improved}
Ishaan Gulrajani, Faruk Ahmed, Martin Arjovsky, Vincent Dumoulin, and Aaron
  Courville.
\newblock Improved training of wasserstein gans.
\newblock {\em arXiv:1704.00028}, 2017.

\bibitem{Argyriou:2008}
Andreas Argyriou, Theodoros Evgeniou, and Massimiliano Pontil.
\newblock Convex multi-task feature learning.
\newblock {\em Mach. Learn.}, 2008.

\bibitem{Bach:2011}
Francis Bach, Rodolphe Jenatton, and Julien Mairal.
\newblock {\em Optimization with Sparsity-Inducing Penalties (Foundations and
  Trends(R) in Machine Learning)}.
\newblock Now Publishers Inc., Hanover, MA, USA, 2011.

\bibitem{VINOD1976147}
H.D. Vinod.
\newblock Canonical ridge and econometrics of joint production.
\newblock {\em Journal of Econometrics}, 1976.

\bibitem{NIPS2007_3340}
Kenji Fukumizu, Arthur Gretton, Xiaohai Sun, and Bernhard Sch\"{o}lkopf.
\newblock Kernel measures of conditional dependence.
\newblock In {\em Advances in Neural Information Processing Systems 20}. 2008.

\bibitem{Beck:2003}
Amir Beck and Marc Teboulle.
\newblock Mirror descent and nonlinear projected subgradient methods for convex
  optimization.
\newblock {\em Oper. Res. Lett.}, 2003.

\bibitem{selfexplainable}
David Alvarez~Melis and Tommi Jaakkola.
\newblock Towards robust interpretability with self-explaining neural networks.
\newblock In {\em Advances in Neural Information Processing Systems 31}. 2018.

\bibitem{Benjamini:1995ram}
Y.~Benjamini and Y.~Hochberg.
\newblock {Controlling the false discovery rate: A Practical and powerful
  approach to multiple testing}.
\newblock {\em J. Roy. Statist. Soc.}, 57:289--300, 1995.

\bibitem{barber2015controlling}
Rina~Foygel Barber, Emmanuel~J Cand{\`e}s, et~al.
\newblock Controlling the false discovery rate via knockoffs.
\newblock {\em The Annals of Statistics}, 43(5):2055--2085, 2015.

\bibitem{mine}
Mohamed~Ishmael Belghazi, Aristide Baratin, Sai Rajeswar, Sherjil Ozair, Yoshua
  Bengio, Aaron Courville, and R~Devon Hjelm.
\newblock Mine: Mutual information neural estimation, 2018.

\bibitem{ozair2019wasserstein}
Sherjil Ozair, Corey Lynch, Yoshua Bengio, Aaron van~den Oord, Sergey Levine,
  and Pierre Sermanet.
\newblock Wasserstein dependency measure for representation learning, 2019.

\bibitem{hastie01statisticallearning}
Trevor Hastie, Robert Tibshirani, and Jerome Friedman.
\newblock {\em The Elements of Statistical Learning}.
\newblock Springer New York Inc., 2001.

\bibitem{Breiman:2001}
Leo Breiman.
\newblock Random forests.
\newblock {\em Mach. Learn.}, 2001.

\bibitem{GradInput}
Avanti Shrikumar, Peyton Greenside, and Anshul Kundaje.
\newblock Learning important features through propagating activation
  differences.
\newblock In {\em Proceedings of the 34th International Conference on Machine
  Learning}, 2017.

\bibitem{saliency}
Karen Simonyan, Andrea Vedaldi, and Andrew Zisserman.
\newblock Deep inside convolutional networks: Visualising image classification
  models and saliency maps.
\newblock {\em International Conference on Learning Representations (Workshop
  Track).}, 2014.

\bibitem{RelvPropagation}
Sebastian Bach, Alexander Binder, Gr{\'e}goire Montavon, Frederick Klauschen,
  Klaus-Robert M{\"u}ller, and Wojciech Samek.
\newblock On pixel-wise explanations for non-linear classifier decisions by
  layer-wise relevance propagation.
\newblock {\em PLoS ONE}, 2015.

\bibitem{feng2017sparse}
Jean Feng and Noah Simon.
\newblock Sparse-input neural networks for high-dimensional nonparametric
  regression and classification.
\newblock {\em arXiv preprint arXiv:1711.07592}, 2017.

\bibitem{liang2018bayesian}
Faming Liang, Qizhai Li, and Lei Zhou.
\newblock Bayesian neural networks for selection of drug sensitive genes.
\newblock {\em Journal of the American Statistical Association}, 113(523),
  2018.

\bibitem{breiman2001random}
Leo Breiman.
\newblock Random forests.
\newblock {\em Machine learning}, 45(1):5--32, 2001.

\bibitem{barretina2012cancer}
Jordi Barretina, Giordano Caponigro, Nicolas Stransky, Kavitha Venkatesan,
  Adam~A Margolin, Sungjoon Kim, Christopher~J Wilson, Joseph Leh{\'a}r,
  Gregory~V Kryukov, Dmitriy Sonkin, et~al.
\newblock The cancer cell line encyclopedia enables predictive modelling of
  anticancer drug sensitivity.
\newblock {\em Nature}, 483(7391):603, 2012.

\bibitem{pedregosa2011scikit}
Fabian Pedregosa, Ga{\"e}l Varoquaux, Alexandre Gramfort, Vincent Michel,
  Bertrand Thirion, Olivier Grisel, Mathieu Blondel, Peter Prettenhofer, Ron
  Weiss, Vincent Dubourg, et~al.
\newblock Scikit-learn: Machine learning in python.
\newblock {\em Journal of machine learning research}, 12(Oct):2825--2830, 2011.

\bibitem{rhee2006genotypic}
Soo-Yon Rhee, Jonathan Taylor, Gauhar Wadhera, Asa Ben-Hur, Douglas~L Brutlag,
  and Robert~W Shafer.
\newblock Genotypic predictors of human immunodeficiency virus type 1 drug
  resistance.
\newblock {\em Proceedings of the National Academy of Sciences},
  103(46):17355--17360, 2006.

\bibitem{hivStanfordTutorial}
Matteo Sesia and Evan Patterson.
\newblock R tutorial for knockoffs - 4.
\newblock
  \url{https://web.stanford.edu/group/candes/knockoffs/software/knockoffs/tutorial-4-r.html},
  2017.

\bibitem{Tseng:2001}
P.~Tseng.
\newblock Convergence of a block coordinate descent method for
  nondifferentiable minimization.
\newblock {\em J. Optim. Theory Appl.}, 109, 2001.

\bibitem{jRazaviyaynHL13}
Meisam Razaviyayn, Mingyi Hong, and Zhi-Quan Luo.
\newblock A unified convergence analysis of block successive minimization
  methods for nonsmooth optimization.
\newblock {\em SIAM Journal on Optimization}, 2013.

\bibitem{van2016conditional}
Aaron Van~den Oord, Nal Kalchbrenner, Lasse Espeholt, Oriol Vinyals, Alex
  Graves, et~al.
\newblock Conditional image generation with pixelcnn decoders.
\newblock In {\em Advances in neural information processing systems}, pages
  4790--4798, 2016.

\bibitem{perez2017learning}
Ethan Perez, Harm de~Vries, Florian Strub, Vincent Dumoulin, and Aaron
  Courville.
\newblock Learning visual reasoning without strong priors.
\newblock {\em arXiv preprint arXiv:1707.03017}, 2017.

\bibitem{kingma2014adam}
Diederik Kingma and Jimmy Ba.
\newblock Adam: A method for stochastic optimization.
\newblock In {\em ICLR}, 2015.

\bibitem{paszke2017automatic}
Adam Paszke, Sam Gross, Soumith Chintala, Gregory Chanan, Edward Yang, Zachary
  DeVito, Zeming Lin, Alban Desmaison, Luca Antiga, and Adam Lerer.
\newblock Automatic differentiation in pytorch.
\newblock In {\em NIPS-W}, 2017.

\bibitem{romano2019}
Yaniv Romano, Matteo Sesia, and Emmanuel Cand{\`e}s.
\newblock Deep knockoffs.
\newblock {\em Journal of the American Statistical Association}, pages 1--27,
  2019.

\end{thebibliography}

%
%
%
\newpage 

\appendix

\section{Algorithms for Convex SIC}\label{alg:ConvexSIC}

\noindent \textbf{Algorithms and Empirical Convex SIC from Samples.}
Given samples from the joint and the marginals,  it is easy to see that the empirical loss $\hat{L}_{\varepsilon}$ can be written in the same way with empirical feature mean embeddings $\hat{\mu}(p_{xy})= \frac{1}{N}\sum_{i=1}^N\Phi_{\omega}(x_i,y_i)$ and $\hat{\mu}(p_{x}p_{y})= \frac{1}{N}\sum_{i=1}^N\Phi_{\omega}(x_i,\tilde{y}_i)$, covariances $\hat{C}(p_{x}p_{y})= \frac{1}{N}\sum_{i=1}^N\Phi_{\omega}(x_i,\tilde{y}_i)\otimes \Phi_{\omega}(x_i,\tilde{y}_i)$ and derivatives grammians $\hat{D}_{j}(p_{x}p_{y})=\frac{1}{N}\sum_{i=1}^N\frac{ \partial \Phi_{\omega}(x_i,\tilde{y}_i)}{\partial x_j } \otimes\frac{ \partial \Phi_{\omega}(x,y)}{\partial x_j } $. Given the strict convexity of $\hat{L}_{\varepsilon}$ jointly in $u$ and $\eta$, alternating optimization as given in Algorithm \ref{alg:Altmin1} in Appendix  is known to be convergent to a global optima (Theorem 4.1 in \citep{Tseng:2001}). Similarly Block Coordinate Descent  (BCD) using first order methods as given in Algorithms \ref{alg:sBCD} and  \ref{alg:BCD} (in Appendix):  gradient descent on $u$ and mirror descent on $\eta$ (in order to satisfy the simplex constraint \citep{Beck:2003}) are also known to be globally convergent   (Theo 2 in \citep{jRazaviyaynHL13}.)

\begin{center}
\vskip -0.2in
\begin{minipage}[t]{0.49\textwidth}
\null 
\begin{algorithm}[H]
\caption{Alternating Optimization}
\begin{algorithmic}
 \STATE {\bfseries Inputs:} $\varepsilon$,$\lambda$, $\tau$, $\rho$ , $\Phi_{\omega}$  \\
 \STATE {\bfseries Initialize} $\hat{\eta}_j=\frac{1}{d_x},\forall j$, $\hat{\delta} = \hat{\mu}(p_{xy}) -\hat{ \mu}(p_xp_y)$ \\
 \FOR{ $i=1\dots \text{Maxiter}$}
\STATE $\hat{u} \gets $
\STATE $\left( \lambda \sum_{j=1}^{d_x} \frac{\hat{D}_{j}(p_{x}p_{y}) }{\hat{\eta}_j} + \rho \hat{C}(p_xp_y)+ \tau I_m \right)^{-1}\hat{\delta}$
\STATE $\hat{\eta}_j \gets  \frac{\sqrt{\scalT{\hat{u}}{\hat{D}_j(p_x p_y) \hat{u}}+\varepsilon}}{ \sum_{k=1}^{d_x}\sqrt{\scalT{\hat{u}}{\hat{D}_k(p_{x}p_{y}) \hat{u}}+\varepsilon}}$
 \ENDFOR
 \STATE {\bfseries Output:} $\hat{u},\hat{\eta}$ 
 \end{algorithmic}
 \label{alg:Altmin1}
\end{algorithm}
\end{minipage}%
\hfill
\begin{minipage}[t]{0.49\textwidth}
\null
\begin{algorithm}[H]
\caption{Block Coordinate Descent }
\begin{algorithmic}
 \STATE {\bfseries Inputs:} $\varepsilon$,$\lambda$, $\tau$, $\rho$, $\alpha,\alpha_{\eta}$ (learning rates),$\Phi_{\omega}$   \\
 \STATE {\bfseries Initialize} $\hat{\eta}_j=\frac{1}{d_x},\forall j$ , $\text{Softmax}(z)= e^{z}/ \sum_{j=1}^{d_x} e^{z_j}$\\
 \FOR{ $i=1\dots \text{Maxiter}$}
\STATE \textcolor{blue}{Gradient step $u$}: 
\STATE $\hat{u}\gets \hat{u} -\alpha \frac{\partial \hat{L}_{\varepsilon}(\hat{u},\hat{\eta})}{\partial u} $
\STATE \textcolor{blue}{Mirror Descent $\eta$} :
\STATE  $\text{logit} \gets \log(\hat{\eta})- \alpha_{\eta} \frac{\partial \hat{L}_{\varepsilon}(\hat{u},\hat{\eta})}{\partial \eta}  $
\STATE $\hat{\eta} \gets \text{Softmax}(\text{logit})$ \COMMENT{stable implementation of softmax}
 \ENDFOR
 \STATE {\bfseries Output:} $\hat{u},\hat{\eta}$ 
 \end{algorithmic}
 \label{alg:BCD}
\end{algorithm}
\end{minipage}
\end{center}
\section{Algorithms for Neural SIC, HRT-SIC and Model-X Knockoff SIC}\label{app:SICNeural}
\begin{center}
\vskip -0.2in
\begin{minipage}[t]{0.49\textwidth}
\null
\begin{algorithm}[H]
\caption{\emph{(non convex)} Neural \textsc{SIC}$(X,Y)$  (Stochastic BCD ) }
\begin{algorithmic}
 \STATE {\bfseries Inputs:} $X,Y$ dataset $X \in \mathbb{R}^{N\times d_x}, Y \in \mathbb{R}^{N\times d_y}$, such that
  $(x_i=X_{i,.},y_i=Y_{i,.}) \sim p_{xy} $  \\
 \STATE {\bfseries Hyperparameters:} $\varepsilon$,$\lambda$, $\tau$, $\rho$, $\alpha_\theta, \alpha_\eta$ (learning rates)   \\
 \STATE {\bfseries Initialize} $\eta_j=\frac{1}{d_x},\forall j$ , $\text{Softmax}(z)= e^{z}/ \sum_{j=1}^{d_x} e^{z_j}$\\
 \FOR{ $iter=1\dots \text{Maxiter}$}
 \STATE Fetch a minibatch of size $N$ $(x_i,y_i) \sim p_{xy}$
 \STATE  Fetch a minibatch of size $N$ $(x_i,\tilde{y}_i) \sim p_{x}p_{y}$ \COMMENT{$\tilde{y}_i$ obtained by permuting rows of $Y$}
 \vskip -02.in
 \STATE \textcolor{blue}{ Stochastic Gradient step on $\theta$}: 
\STATE $\theta \gets  \theta -\alpha_\theta \frac{\partial \hat{L}(f_{\theta},\eta)}{\partial \theta} $\COMMENT{We use ADAM}
\STATE \textcolor{blue}{Mirror Descent $\eta$} :
\STATE  $\text{logit} \gets \log(\eta)- \alpha_\eta \frac{\partial \hat{L}(f_{\theta},\eta)}{\partial \eta}  $
\STATE $\eta \gets \text{Softmax}(\text{logit})$ \COMMENT{stable implementation of softmax}
 \ENDFOR
 \STATE {\bfseries Output:} $f_{\theta},\eta$ 
 \end{algorithmic}
 \label{alg:sBCD}
\end{algorithm}
\end{minipage}
\hfill
\begin{minipage}[t]{0.49\textwidth}
\null 
\begin{algorithm}[H]
\caption{ HRT With SIC $(X,Y)$ }
\begin{algorithmic}
\STATE  {\bfseries Inputs:} $ D_{train}=(X_{tr},Y_{tr})$ , a Heldout set $D_{\text{Holdout}}=(X,Y)$, features Cutoff $K$ \\
\STATE {\bfseries SIC:}  $(f_{\theta^*},\eta_*) = \textsc{SIC}( D_{train})$ \COMMENT{Alg. \ref{alg:sBCD}}
\STATE {\bfseries Score of witness on Hold out :} $S^* = \textsc{mean}(f_{\theta^*}(X,Y))$
\STATE {\bfseries Conditional Generators} Pre-trained conditional Generator : $G(x_{-j},j)$ predicts $X_j|X_{-j}$ 
\vskip -0,2 in
 \STATE {\bfseries  Shortlist :} $I = \textsc{IndexTopK}(\eta) $
 \STATE \COMMENT{$p$-values for  $j \in I $; randomizations tests}
\FOR{ $j\in  I$}
\FOR {$r =1\dots R$}
\STATE Construct $\tilde{X}$, $\tilde{X}_{.,k}= X_{.,k}  \forall k \neq j$ and $\tilde{X}_{.,j}= G(X_{-j},j)$ \COMMENT{Simulate  Null Hyp.}
\STATE  $ S_{j,r} = \textsc{mean}(f_{\theta^*}(\tilde{X},Y))$ \COMMENT{Score of witness function on the Null}
\ENDFOR
\STATE $p_{j}= \frac{1}{R+1}\left(1+ \sum_{r=1}^R 1_{S^r_j\geq  S^*} \right)$ 
 \ENDFOR 
 \STATE  discoveries =\textbf{\textsc{BH}}\text{(p,targetFDR)} \COMMENT{Benjamini-Hochberg Procedure}
 \STATE {\bfseries Output:} discoveries
 \end{algorithmic}
 \label{alg:hrt_sic}
\end{algorithm}
\end{minipage}%

\end{center}

\begin{minipage}[t]{0.66\textwidth}
\null 
\begin{algorithm}[H]
\caption{Model-X Knockoffs FDR control with SIC }
\begin{algorithmic}
\STATE  {\bfseries Inputs:} $ D_{train}=(X_{tr},Y_{tr})$ , Model-X knockoff features $\tilde{X}\sim \mathrm{ModelX}(X_{tr})$, target FDR $q$\\
\STATE {\bfseries Train SIC:}  $(f_{\theta^*},\eta) = \textsc{SIC}([X_{tr}, \tilde{X}], Y)$, \COMMENT{Alg. \ref{alg:sBCD}}
where $[X_{tr}, \tilde{X}]$ is the concatenation of $X_{tr}$ and knockoffs $\tilde{X}$
\FOR{ $j=1,\ldots,d_X$}
\STATE {Compute importance score of $j$ feature:} $W_j=\eta_j - \eta_{j+d_x}$, where $\eta_{j+d_X}$ is the $\eta$ of feature knockoff $\tilde{X}_j$
\ENDFOR
\STATE {Compute threshold $\tau>0$ by setting} \\$\tau = \min\left\{t>0:\frac{\#\{j:W_j\leq-t\}}{\#\{j:W_j\geq t\}}\le q\right\}$
 \STATE {\bfseries Output:} discoveries $\left\{j : W_j>\tau  \right\}$
 \end{algorithmic}
 \label{alg:modelx_sic}
\end{algorithm}
\end{minipage}%

\section{Proofs}

\begin{proof}[Proof of Theorem \ref{theo:ConvexityAndSmoothing}]
1) Let $\delta=  \mu(p_{xy})- \mu(p_xp_y)$.

We have $$ L(u,\eta)= -\scalT{u}{\delta} +\frac{1}{2}\scalT{u}{( \rho C(p_xp_y)+ \tau I_m)u} + \frac{\lambda}{2} \sum_{j} \frac{\scalT{u}{D_j(p_{x}p_{y}) u}}{\eta_j} , u \in \mathbb{R}^m \text{ and } \eta \in \Delta^{d_x}$$
where $\Delta^{d_x}$ is the probability simplex.
$L$ is the sum of a linear tem and quadratic terms (convex in $u$) and a  function of the form 
 $$f(u,\eta)=\frac{1}{2} \sum_{j=1}^{d_{x}} \frac{u^{\top}A_ju }{\eta_j}$$
where $A_j$ are PSD matrices, and $\eta$ is in the  probability simplex (convex). Hence  it is enough to show that $f$ is jointly convex.  
 Let $g(w,\eta)=\frac{w^{\top}Aw }{\eta}, \eta >0 $. The Hessian of $g(w,\eta)$, $Hg$ has the following form: 
\[
Hg(w,\eta)=
\left[
\begin{array}{c|c}
\frac{\partial^2 L}{\partial w\otimes \partial w} & \frac{\partial^2 L}{\partial w  \partial \eta}\\ \hline
\frac{\partial^2 L}{\partial \eta \partial w} &  \frac{\partial^2 L}{\partial \eta ^2 } 

\end{array}
\right]=
\left[
\begin{array}{c|c}
\frac{A}{\eta}  & -\frac{Aw}{\eta^2} \\ \hline
-\frac{w^{\top}A}{\eta^2} & \frac{w^{\top}Aw}{\eta^3}
\end{array}
\right]
\]
Let us prove that for all $(w,\eta), \eta_j>, \forall j0$:
 $$(w',\eta')^{\top}Hg(w,\eta)(w',\eta')\geq 0,\forall (w',\eta'), \eta'_j >0 ,\forall j$$
 We have :
 \begin{eqnarray*}
 (w',\eta')^{\top}Hg(w,\eta)(w',\eta')&=&\frac{\scalT{w'}{Aw'}}{\eta}-2 \eta' \frac{\scalT{w'}{Aw}}{\eta^2}+ \eta'^2\frac{w^{\top}Aw}{\eta^3}\\
 &=& \frac{1}{\eta}\left(\scalT{w'}{Aw'}-\frac{2\eta'}{\eta}\scalT{w'}{Aw}+ \frac{\eta'^2}{\eta^2} w^{\top}Aw\right)\\
 &=& \frac{1}{\eta} \nor{A^{\frac{1}{2}} w'- \frac{\eta'}{\eta} A^{\frac{1}{2}}w }^2_{2}\geq 0 \text{ for } \eta>0
 \end{eqnarray*}
 Now back to $f$ it is easy to see that : 
$$ (w',\eta')^{\top}Hf(w,\eta)(w',\eta')=\sum_{j=1}^{d_x} \frac{1}{\eta_j} \nor{A_j^{\frac{1}{2}} w'- \frac{\eta'_j}{\eta_j} A_j^{\frac{1}{2}}w }^2_{2}\geq 0 \text{ for } \eta \in \Delta^{d_x}_j,\eta_j>0.$$
 Hence the loss $L$ is jointly  convex in $(u,\eta)$. Due to discontinuity at $\eta_j=0$ the loss is not continuous .\\
  
\noindent 2)  It is easy to see that the hessian becomes definite:
$$ (w',\eta')^{\top}HL_{\varepsilon}(w,\eta)(w',\eta')=\sum_{j=1}^{d_x} \frac{1}{\eta_j} \left( \nor{A_j^{\frac{1}{2}} w'- \frac{\eta'_j}{\eta_j} A_j^{\frac{1}{2}}w }^2_{2} + \varepsilon (\frac{\eta'_j}{\eta_j})^2 \right)>  0 \text{ for } \eta \in \Delta^{d_x}_j,\eta_j,\eta'_j>0,$$
and $L_{\varepsilon}(u,\eta)$ is jointly strictly convex, $u$ is unconstrained and $\eta$ belongs to a convex set (the probability simplex)   and hence admits a unique minimizer.

\noindent 3) The unique minimizer  satisfies first order optimality conditions for the following Lagragian:
$$ \mathcal{L}(u,\eta, \xi)= L_{\varepsilon}(u,\eta)+\xi(\sum_j \eta_j -1) $$
$$ \frac{\partial \mathcal{L}(u,\eta,\xi)}{\partial u}= -\delta +  \left( \lambda \sum_{j=1}^{d_x} \frac{D_{j}(p_{x}p_{y}) }{\eta_j} + \rho C(p_xp_y)+ \tau I_m \right) u=0  $$
and 
$$\frac{\partial \mathcal{L}(u,\eta,\xi)}{\partial \eta_j}= -\frac{\lambda}{2}\frac{\scalT{u}{D_j(p_xp_y) u} + \varepsilon }{\eta^2_j} +\xi =0 $$
and 
$$\frac{\partial \mathcal{L}(u,\eta,\xi)}{\partial \xi} = \sum_{j}\eta_j-1=0$$
Hence:
$$u^*_{\varepsilon}= \left( \lambda \sum_{j=1}^{d_x} \frac{D_{j}(p_{x}p_{y}) }{\eta^*_j} + \rho C(p_xp_y)+ \tau I_m \right)^{-1}(\mu(p_{xy}) - \mu(p_xp_y) )$$
and :
$$\eta_{j,\varepsilon}^* = \frac{\sqrt{\scalT{u^*_{\varepsilon}}{D_j(p_x p_y) u^*_{\varepsilon}}+\varepsilon}}{ \sum_{k=1}^{d_x}\sqrt{\scalT{u^*_{\varepsilon}}{D_k(p_{x}p_{y}) u^*_{\varepsilon}}+\varepsilon}}.$$

\end{proof}

\begin{proof}[Proof of Theorem \ref{theo:SICoptim}] The proof follows similar proof in Argryou 2008.
$$S_{\varepsilon}(u)= L(u_{\varepsilon}, \eta(u_{\varepsilon}))= -\scalT{u}{\delta} +\frac{1}{2}\scalT{u}{( \rho C(p_xp_y)+ \tau I_m)u} + \frac{\lambda}{2} \left( \sum_{j} \sqrt{\scalT{u}{D_j(p_{x}p_{y}) u}+\varepsilon}\right)^2 $$
Let $\{(u_{\ell_n},\eta_{\ell_n}(u_{\ell_n})), n \in \mathbb{N}\}$ be  a limiting subsequence of minimizers of $L_{\varepsilon_{\ell_n}}(.,.)$ and let $(u^*,\eta^*)$ be its limit as $n\to \infty$. From the definition of $S_{\varepsilon}(u)$, we see that $\min_{u} S_{\varepsilon}(u)$  decreases as $\varepsilon$ decreases to zero, and admits a limit $\bar{S}= \min_{u}S_{0}(u)$. Hence $S_{\varepsilon_{\ell_n}} \to \bar{S}$. Note that $S_{\varepsilon}(u)$ is continuous in both $\varepsilon$ and $u$ and we have finally $S_0(u^*)=\bar{S}$, and $u^*$ is a minimizer of $S_0$.
\end{proof}

\begin{proof}[Proof of Corollary \ref{cor:InterpretableSIC} ]   The optimum $(u^*_{\varepsilon},\eta^*_{\varepsilon})$  satisfies:
$$-\delta +  \left( \lambda \sum_{j=1}^{d_x} \frac{D_{j}(p_{x}p_{y}) }{\eta_j} + \rho C(p_xp_y)+ \tau I_m \right) u^*_{\varepsilon}=0$$
Let $f^*(x)=\scalT{u}{\Phi_{\omega}(x,y)}$ and define $||f^*_{\varepsilon}||_{\mathcal{F}}=\nor{u^*_{\varepsilon}}_2$. It follows that  $\eta^*_j= \frac{\sqrt{\mathbb{E}_{p_{x}p_{y}} \left|\frac{\partial f^*_{\varepsilon}(x,y)}{\partial x_j}\right|^2 +\varepsilon}}{\sum_{k}\sqrt{\mathbb{E}_{p_{x}p_{y}} \left|\frac{\partial f^*_{\varepsilon}(x,y)}{\partial x_k}\right|^2 +\varepsilon}}  $
\begin{eqnarray*}
&\text{Note that we have} &\mathbb{E}_{p_{xy}}f^*_{\varepsilon}(x,y)- \mathbb{E}_{p_xp_y}f^*_{\varepsilon}(x,y)\\
&=&\scalT{\delta}{u^*_{\varepsilon}}\\
&=&\scalT{u^*_{\varepsilon}}{\left( \lambda \sum_{j=1}^{d_x} \frac{D_{j}(p_{x}p_{y}) }{\eta^*_{j,\varepsilon}} + \rho C(p_xp_y)+ \tau I_m \right) u^*_{\varepsilon}}\\
&=&\lambda \left(\sum_{j=1}^{d_x} \sqrt{\mathbb{E}_{p_xp_y} |\frac{\partial f^*_{\varepsilon}(x,y)}{\partial x_j}|^2 +\varepsilon}\right)^2+ \rho \mathbb{E}_{p_xp_y} f^{*,2}_{\varepsilon}(x,y) + \tau ||f^*_{\varepsilon}||^2_{\mathcal{F}} 
\end{eqnarray*}

\begin{eqnarray*}
  SIC_{(L^1)^2,\varepsilon} &=& \mathbb{E}_{p_{xy}}f^*_{\varepsilon}(x,y)- \mathbb{E}_{p_xp_y}f^*_{\varepsilon}(x,y) - \frac{1}{2}( \lambda \left(\sum_{j=1}^{d_x} \sqrt{\mathbb{E}_{p_xp_y} |\frac{\partial f^*_{\varepsilon}(x,y)}{\partial x_j}|^2 +\varepsilon}\right)^2\\
  &+& \rho \mathbb{E}_{p_xp_y} f^{*,2}_{\varepsilon}(x,y) + \tau ||f^*_{\varepsilon}||^2_{\mathcal{F}} )\\
  &=&\frac{\lambda}{2} \left(\sum_{j=1}^{d_x} \sqrt{\mathbb{E}_{p_xp_y} |\frac{\partial f^*_{\varepsilon}(x,y)}{\partial x_j}|^2 +\varepsilon}\right)^2 +\frac{\rho}{2} \mathbb{E}_{p_xp_y} f^{*,2}_{\varepsilon}(x,y) + \frac{\tau}{2} ||f^*_{\varepsilon}||^2_{\mathcal{F}}   \\ 
  &=& \frac{1}{2}\left(\mathbb{E}_{p_{xy}} f^*_{\varepsilon}(x,y) -  \mathbb{E}_{p_{x}p_{y}} f^*_{\varepsilon}(x,y)\right)
 \end{eqnarray*}

We conclude by taking $\varepsilon \to 0$.
\end{proof}



\section{FDR Control with HRT and Conditional Generative Models}
\label{sec:Generators}

The Holdout Randomization Test (HRT) is a principled method to produce valid $p$-values for each feature, that enables the control over the false discovery of a predictive model \citep{HRT}.
The $p$-value associated to each feature $x_j$ essentially quantifies the result of a conditional independence test with the null hypothesis stating that $x_j$ is independent of the output $y$, conditioned on all the remaining features $\mathbf{x}_{-j}=(x_1,\ldots,x_{j-1}, x_{j+1}, \ldots,x_p)$.
This in practice requires the availability of an estimate of the complete conditional of each feature $x_j$, i.e.\ of $P(x_j|\mathbf{x}_{-j})$.
HRT then samples the values of $x_j$ from this conditional distribution to obtain the $p$-value associated to it.
Taking inspiration from neural network models for conditional generation (see e.g.\ \cite{van2016conditional}) we train a neural network to act as a generator of a features $x_j$ given the remaining features $\mathbf{x}_{-j}$ as inputs, as a replacement for the conditional distributions $P(x_j|\mathbf{x}_{-j})$.
In all of our tasks, one three-layer neural network with 200 ReLU units and Conditional Batch Normalization (BCN) \cite{perez2017learning} applied to all hidden layers serves as generator for all features $j=1,\ldots,p$.
A sample from $P(x_j|\mathbf{x}_{-j})$ is generated by giving as input to the network an index $j$ indicating the feature to generate, and a sample $\mathbf{x}_{-j} \sim P(\mathbf{x}_{-j})$, represented as a sample from the full joint distribution $\mathbf{x} \sim P(x_1,\ldots,x_p)$, with feature $j$ being masked out.
In practice, the index $j$ and $\mathbf{x} \sim P(x_1,\ldots,x_p)$ are given as inputs to the generator, and the neural network model does the masking, and sends the index $j$ to the CBN modules which normalize their inputs using $j$-dependent centering and normalization parameters.
The output of the generator is a $n_{bins}$-dimensional softmax over bins tessellating the range of the distribution of $x_j$, such that the bins are uniform quantiles of the inverse CDF of the distribution of $x_j$ estimated over the training set.
In all simulations we used a number of bins $n_{bins}=100$.

Generators are trained randomly sampling an index $j=1,\ldots,p$ for each sample $\mathbf{x}$ in the training set, and minimizing the cross-entropy loss between the output of the generator neural network $Gen(j, \mathbf{x})$ and $x_j$ using mini-batch SGD.
In particular, we used the Adam optimizer \citep{kingma2014adam} with the default pytorch \citep{paszke2017automatic} parameters and learning rate $\lambda=0.003$ which is halved every 20 epochs, and batch size of 128.

\section{Discussion of SIC: Consistency, Computational Complexity and FDR Control}
\paragraph{SIC consistency.} In order to recover the correct conditional independence we elected to use FDR control techniques to perform those dependent hypotheses testing (btw coordinates). By combining SIC with HRT and knockoffs we can guarantee that the correct dependency is recovered while the FDR is under control. For the consistency of SIC in the classical sense, one needs to analyze the solution of SIC, when the critic is not constrained to belonging to an RKHS. This can be done by studying the solution of the equivalent PDE corresponding to this problem (which is challenging, but we think can also be managed through the $\eta$- trick). Then one would proceed by finding 1) conditions under which this solution exists in the RKHS, 2)  generalization bounds from samples to the population solution in the RKHS. We leave this analysis for future work.
\paragraph{Computational Complexity of Neural SIC.} The cost of training SIC with SGD and mirror descent has the same scaling in the size of the problem as training the base regressor neural network via back-propagation. The only additional overhead is the gradient penalty, where the cost is that of double back-propagation. In our experiments, this added computational cost is not an issue when training is performed on GPU.
\paragraph{SIC-HRT versus SIC-Knockoffs.}
For a comparison between HRT and knockoffs, we refer the reader to \citep{HRT}, which shows similar performance for either method in terms of controlling FDR. Each method has its advantages. In HRT most of the computation is in 1) training the generative models, and 2) performing the randomization test, i.e.\ forwarding the data through the critic and computing $p$-values for each coordinate for $R$ runs. On the other hand, if knockoff features can be modelled as a multivariate Gaussian, controlling FDR with knockoffs can be done very cheaply, since it does not require randomization tests. If instead knockoff features have to be generated through nonlinear models, knockoffs can be computationally expensive as well (see for example \citep{romano2019}). 

\section{Experimental details}
\label{app:expdetails}
\subsection{Synthetic Datasets}
\label{subapp:datadetails}

\subsubsection{Complex Multivariate Synthetic Dataset (SinExp)}
The SinExp dataset is generated from a complex multivariate model proposed in \cite{feng2017sparse} Sec 5.3, where 6 features $x_i$ out of 50 generate the output $y$ through a non-linearity involving the product and composition of the $\cos$, $\sin$ and $\exp$ functions, as follows:
\begin{align*}
y & = \sin(x_1 (x_1 + x_2)) \cos(x_3 + x_4 x_5)\sin(e^{x_5} + e^{x_6} - x_2).
\end{align*}
We increase the difficulty even further by introducing a pairwise correlation between all features of 0.5. We perform experiments using datasets of 125 and 500 samples. 
For each sample size, 100 independent datasets are generated.

\subsubsection{Liang Dataset}
\emph{Liang Dataset} is a variant of the synthetic dataset proposed by \cite{liang2018bayesian}. 
The dataset prescribes a regression model with 500-dimensional correlated input features $x$, where the 1-D regression target $y$ depends on the first 40 features only (the last 460 correlated features are ignored). 
In the original dataset proposed by \cite{liang2018bayesian}, $y$ depends on $4$ features only,
this more complex version of the dataset that uses $40$ features was proposed by \cite{HRT}.
The target $y$ is computed as follows:
\begin{equation}
\label{eqn:liang_40}
y = \sum_{j=0}^{9} \left[ w_{4j} x_{4j} + w_{4j+1} x_{4j+1} + \text{tanh}(w_{4j+2}x_{4j+2} + w_{4j+3}x_{4j+3})\right] + \sigma\epsilon \, ,
\end{equation}
with $\sigma = 0.5$ and $\epsilon \sim \mathcal{N}(0,1)$. 
As in \cite{HRT},
the $500$ features are generated to have $0.5$ correlation coefficient with each other,
\begin{equation}
\label{eqn:benchmark_covariates}
x_j = (\rho + z_j) / 2\, , \quad j = 1, \ldots, 500 \, ,
\end{equation}
where $\rho$ and $z_j$ are independently generated from $\mathcal{N}(0,1)$.

Our experimental results are the average over 100 generated datasets, each consisting of 500 train and 500 heldout samples.
\begin{figure}[ht!]
  \centering
  \subfloat{\includegraphics[scale=0.32]{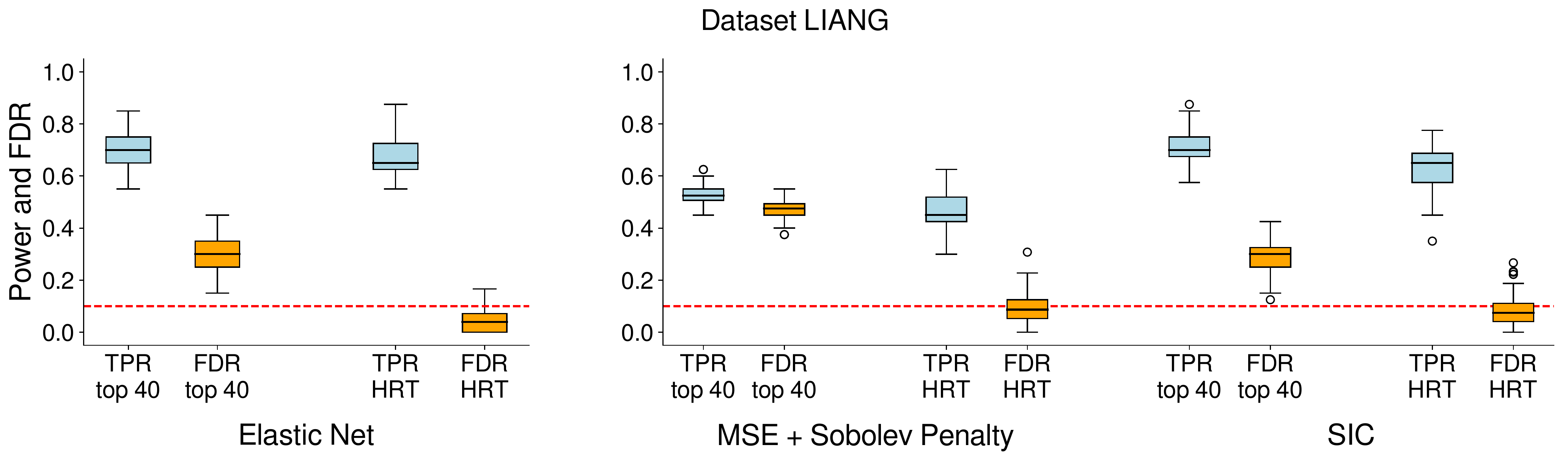}}\\
  \caption{Liang synthetic dataset. TPR and FDR of Elastic Net baseline models (left panels) are compared against our methods, analogously to Fig.\ \ref{fig:sinexp}.
  Differently from Fig.\ \ref{fig:sinexp}, however, TPR and FDR are computed by selecting the top 40 features in order of importance (since this dataset was generated with 40 ground truth features).
  Moreover, HRT is used to select features among the top 100 most important features.}
  \label{fig:liang}
\end{figure}


\subsection{CCLE Dataset}
\label{subapp:ccle}
The Cancer Cell Line Encyclopedia (CCLE) dataset \cite{barretina2012cancer} provides data about of anti-cancer drug response in cancer cell lines.
The dataset contains the phenotypic response measured as the area under the dose-response curve (AUC) for a variety of drugs that were tested against hundreds of cell
lines.
\cite{barretina2012cancer} analyzed each cell to obtain gene mutation and expression features.
The total number of data points (cells) is 479.
We followed the preprocessing steps by \cite{HRT} and 
first screened the genomic features to filter out features with less than 0.1 magnitude Pearson correlation
to the AUC. This resulted in a final set of about 7K features.
The main goal in this task is to discover the genomic features associated with drug response.
Following \cite{HRT},
we perform experiments for the drug PLX4720.
Table \ref{tab:ccle_ftrs} presents the top-10 genomic features selected by SIC according to $\eta_j$ values. In Sec. \ref{sec:exp}, we also present quantitative results that show the effectiveness of these features when used to train regression models. 

\begin{table}[h]
\centering
\begin{tabular}{llr}
\toprule
{} &            Genomic Feature &  $\eta_j$ \\
\midrule
0 &    \verb$BRAF.V600E_MUT$ * &  0.011837 \\
1 &              \verb$ACKR3$  &  0.011712 \\
2 &       \verb$RP11-349I1.2$  &  0.010534 \\
3 &  \verb$BRAF_MUT$ $\dagger$ &  0.010449 \\
4 &           \verb$UBE2V1P5$  &  0.010420 \\
5 &            \verb$EPB41L3$  &  0.010163 \\
6 &  \verb$C11orf85$ $\dagger$ &  0.009622 \\
7 &       \verb$RP11-395F4.1$  &  0.009449 \\
8 &           \verb$SERPINA9$  &  0.009387 \\
9 &          \verb$RN7SKP281$  &  0.009369 \\
\bottomrule
\end{tabular}
\caption{Top-10 Genomic Features selected by SIC according to $\eta_j$ values. These are the most important features for high mutual information with PLX4720 response variable, on the CCLE dataset.
* indicates feature also discovered by Elastic Net and HRT \cite{HRT}. $\dagger$ indicates feature also discovered by Elastic Net in original CCLE paper \cite{barretina2012cancer}. 
}
\label{tab:ccle_ftrs}
\end{table}

\subsection{SIC Neural Network descriptions and training details}
\label{app:nndetails}

The first critic network used in the experiments (with SinExp and HIV-1 datasets) is a standard three-layer ReLU dropout network with no biases, i.e. small\_critic.
When using this network, the inputs $X$ and $Y$ are first concatenated then given as input to the network.
The two first layers have size 100, while the last layer has size 1.
We train the network using Adam optimizer with $\beta_1=0.5$, $\beta_2=0.999$, weight\_decay=1e-4 learning rate $\alpha_\eta$ = 1e-3 and $\alpha_\eta=0.1$, and perform 4000 training iterations/updates, computed with batches of size $100$.
All NNs used in our experiments were implemented using PyTorch \cite{paszke2017automatic}.
\begin{verbatim}
small_critic(
  (branchxy): Sequential(
    (0): Linear(in_features=51, out_features=100, bias=False)
    (1): ReLU()
    (2): Dropout(p=0.3)
    (3): Linear(in_features=100, out_features=100, bias=False)
    (4): ReLU()
    (5): Dropout(p=0.3)
    (6): Linear(in_features=100, out_features=1, bias=False)
  )
)
\end{verbatim}

The critic network used in the experiments with Liang and CCLE datasets contains two different branches that separately process the inputs $X$ (\emph{branchx}) and $Y$ (\emph{branchy}), then the output of these two branches are concatenated and processed by a final branch that contains three-layer LeakyReLU network (\emph{branchxy}). We name this network \emph{big\_critic} (see figure bellow for details about layer sizes). This network is trained with the same Adam settings as above for 4000 updates (Liang) and 8000 updates (CCLE).

\begin{verbatim}
big_critic(
  (branchx): Sequential(
    (0): Linear(in_features=500, out_features=100, bias=True)
    (1): LeakyReLU(negative_slope=0.01)
    (2): Linear(in_features=100, out_features=100, bias=True)
    (3): LeakyReLU(negative_slope=0.01)
  )
  (branchy): Sequential(
    (0): Linear(in_features=1, out_features=100, bias=True)
    (1): LeakyReLU(negative_slope=0.01)
    (2): Linear(in_features=100, out_features=100, bias=True)
    (3): LeakyReLU(negative_slope=0.01)
  )
  (branchxy): Sequential(
    (0): Linear(in_features=200, out_features=100, bias=True)
    (1): LeakyReLU(negative_slope=0.01)
    (2): Linear(in_features=100, out_features=100, bias=True)
    (3): LeakyReLU(negative_slope=0.01)
    (4): Linear(in_features=100, out_features=1, bias=True)
  )
)
\end{verbatim}

The regressor NN used for the downstream regression task in Section \ref{sec:ccle} is a standard three-layer ReLU dropout network.
This regressor NN was trained with the same Adam settings as above for 1000 updates with a batchSize of 16. We did not perform any hyperparameter tuning or model selection on heldout MSE performance.
\begin{verbatim}
regressor_NN(
  (net): Sequential(
    (0): Linear(in_features=7251, out_features=100, bias=True)
    (1): ReLU()
    (2): Dropout(p=0.3)
    (3): Linear(in_features=100, out_features=100, bias=True)
    (4): ReLU()
    (5): Dropout(p=0.3)
    (6): Linear(in_features=100, out_features=1, bias=True)
  )
)
\end{verbatim}


\end{document}